\documentclass{article}

     \usepackage[preprint, nonatbib]{neurips_2022}

\usepackage[utf8]{inputenc} 
\usepackage[T1]{fontenc}    
\usepackage{url}            
\usepackage{booktabs}       
\usepackage{amsfonts}       
\usepackage{nicefrac}       
\usepackage{microtype}      
\usepackage{xcolor}         

\usepackage{microtype}
\usepackage{graphicx}
\usepackage{subfigure}
\usepackage{booktabs} 
\usepackage{amsfonts}       
\usepackage{amsmath, amssymb}
\usepackage{mathtools}
\usepackage{amsthm}
\usepackage{color}
\usepackage{hyperref}       
\usepackage{tikz}
\usetikzlibrary{matrix,chains,positioning,decorations.pathreplacing,arrows}

\usepackage{hyperref}

\usepackage{amsmath}
\usepackage{amssymb}
\usepackage{mathtools}
\usepackage{amsthm}

\newcommand{\R}{\mathbb{R}}
\newcommand{\E}{\mathbb{E}}

\newcommand{\N}{\mathcal{N}}

\newcommand{\hba}{\hat{\beta}_A}
\newcommand{\hbb}{\hat{\beta}_B}
\newcommand{\hbba}{\hat{\beta}_{BA}}
\newcommand{\hb}{\hat{\beta}}

\newcommand{\pba}{\mathcal{P}_{B^\top} \hba}

\newcommand{\eps}{\epsilon}
\newcommand{\s}{\sigma}
\newcommand{\gam}{\gamma}

\theoremstyle{plain}
\newtheorem{theorem}{Theorem}[section]
\newtheorem{proposition}[theorem]{Proposition}
\newtheorem{lemma}[theorem]{Lemma}

\theoremstyle{definition}

\newtheorem{assumption}[theorem]{Assumption}
\theoremstyle{remark}

\title{Analysis of Catastrophic Forgetting for Random Orthogonal Transformation Tasks in the Overparameterized Regime}

\author{
  Daniel Goldfarb \\
  Khoury College of Computer Sciences\\
  Northeastern University\\
  Boston, MA 02115 \\
  \texttt{goldfarb.d@northeastern.edu} \\
  \And
  Paul Hand \\
  Dept. of Mathematics and Khoury College of Computer Sciences\\
  Northeastern University\\
  Boston, MA 02115 \\
  \texttt{p.hand@northeastern.edu} \\
}

\begin{document}

\maketitle

\begin{abstract}
Overparameterization is known to permit strong generalization performance in neural networks. In this work, we provide an initial theoretical analysis of its effect on catastrophic forgetting in a continual learning setup.  We show experimentally that in permuted MNIST image classification tasks, the generalization performance of multilayer perceptrons trained by vanilla stochastic gradient descent can be improved by overparameterization, and the extent of the performance increase achieved by overparameterization is comparable to that of state-of-the-art continual learning algorithms.  We provide a theoretical explanation of this effect by studying a qualitatively similar two-task linear regression problem, where each task is related by a random orthogonal transformation. We show that when a model is trained on the two tasks in sequence without any additional regularization, the risk gain on the first task is small if the model is sufficiently overparameterized.



\end{abstract}

\section{Introduction}
\label{intro}

Continual learning is the ability of a model to learn continuously from a stream of data, building on what was previously learned and retaining previously learned skills without the need for retraining. A major obstacle for neural networks to learn continually is the catastrophic forgetting problem: the abrupt drop in performance on previous tasks upon learning new ones. Modern neural networks are typically trained to greedily minimize a loss objective on a training set, and without any regularization, the model's performance on a previously trained task may degrade. Techniques for mitigating catastrophic forgetting fall under three main groups: generative replay, parameter isolation, and regularization methods \cite{delange2021continual}. Generally, the goal of regularization methods is to determine important parameters from previous tasks and constrain them so that they do not get modified too much while training subsequent tasks.  Two common regularization methods are Synaptic Intelligence (SI) \cite{zenke2017continual} and Elastic Weight Consolidation (EWC) \cite{kirkpatrick2017overcoming}.   See Appendix \ref{cl-techniques} for a detailed description of them.


It is well-known that strong generalization performance for neural networks is typically obtained in the overparameterized regime, where the number of learnable parameters is greater than the number of training examples. Work on overparameterized machine learning has led to research on the so-called double descent phenomenon, where test error improves as model complexity increases beyond the level needed to fit the training data,
 outperforming all underparameterized versions of the model \cite{belkin2019reconciling}. One of the first observations of this behavior in modern neural networks was in extremely wide ResNet18 models that generalize better than their underparameterized counterparts on CIFAR-10 despite fitting to label noise \cite{nakkiran2021deep}. This model-wise double descent phenomenon has been demonstrated analytically in a variety of machine learning models \cite{hastie2019surprises, belkin2020two, bartlett2020benign}, including some as simple as linear regression.  Such linear models will be the basis of theoretical analysis in the present paper.

While investigating the relationship of overparameterization with catastrophic forgetting, we  made an interesting experimental observation. We considered 10 permuted-MNIST tasks, where each task has training data given by random permutations of the original MNIST images \cite{lecun1998mnist}.   We trained  2-layer multilayer perceptrons (MLPs) with a variety of layer widths given by $[400w, 400w]$, where $w = 1, 3, 5, 7, 9$, using vanilla stochastic gradient descent and the continual learning algorithms, SI and EWC.  We compared the test accuracy on all seen tasks.  As expected, and as shown in Figure \ref{nn-experiments}, average accuracy with SGD drops significantly after learning multiple tasks, and that drop is mitigated by using SI or EWC.  Interestingly, we observe that a significant fraction of the accuracy gain achieved by SI or EWC can be obtained with vanilla SGD by simply overparameterizing the model.   This can be seen by comparing the $w=1$ and $w=9$ curves with SGD to the curves with SI and EWC.  See Appendix \ref{mnist-appendix} for more details on the experiments.


\begin{figure}
\begin{center}
    \includegraphics[scale=0.42]{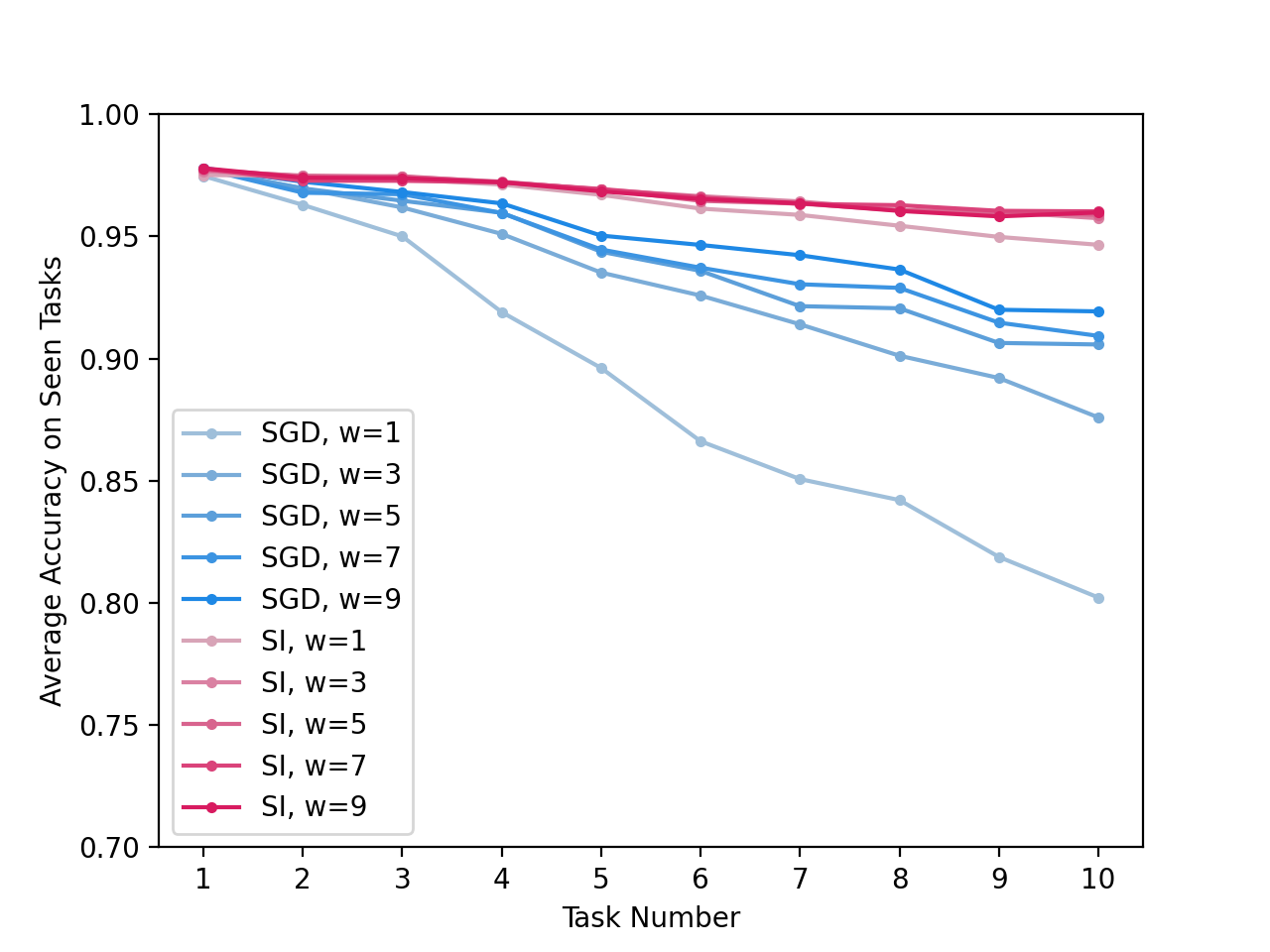}
    \includegraphics[scale=0.42]{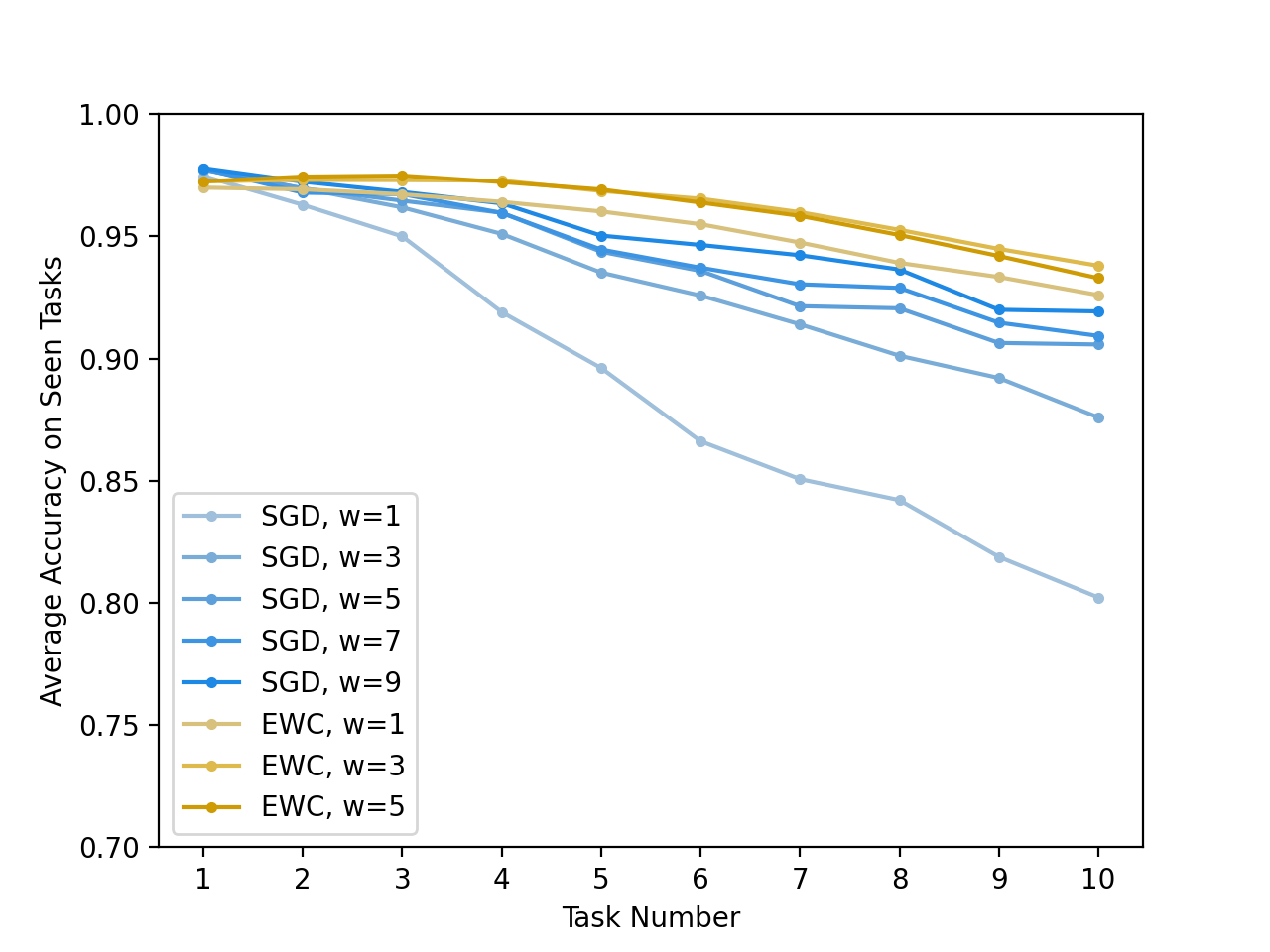}
    \caption{Results of permuted MNIST experiment. Red curves denote performance of SI, yellow curves denote performance of EWC, blue curves denote performance of Vanilla SGD. Bolder saturation of lines corresponds to larger width parameters (more overparameterization). Specific hyperparameters are reported in Appendix \ref{mnist-appendix}. Curves for $w=7,9$ for EWC are omitted due to computational constraints on Fisher matrix estimates.}
    \label{nn-experiments}
\end{center}
\end{figure}


The goal of the present paper is to analytically illustrate the effect overparameterization can have on catastrophic forgetting. As with the illustrations of double descent in \cite{hastie2019surprises, belkin2020two, bartlett2020benign}, we choose to study the effect with a linear regression problem for simplicity and mathematical convenience. 
While the experiments above study the case of tasks related by random permutations, our analysis will instead study the qualitatively similar case where two tasks are related by a random orthogonal transformation, which results in simpler mathematical analysis, as we discuss in Section \ref{discussion}. The relation given by the random orthogonal transformation gives two data feature spaces corresponding to each task that are approximately but not exactly orthogonal. We construct two tasks, A and B.  Let task A be defined by data matrix $X_A \in \R^{n \times p}$, with rows being $p$ noisy random projections of some low-dimensional latent features, and responses $y\in \R^n$ that are noiseless and linear in the latent features.   
Let $O$ be a random $p \times p$ orthogonal matrix. Then task B is defined by data $X_B=X_AO^\top$ and the same responses $y$. Learning tasks A and B involve estimating a $\beta \in \R^p$  for the predictor $f_\beta: x \mapsto x^\top \beta$ to fit the data $(X_A, y)$ and $(X_B, y)$, respectively. 

We analyze the increase in statistical risk on task A between an estimator trained on task $A$ by minimizing square loss, with initialization at zero, and one sequentially trained on task $A$ and then task $B$ with no explicit regularization. Let $R(f_{\beta})$ be the risk on task A of an estimator $f$ with parameters $\beta$. Let $\hba$ be the parameters of the model that is trained on task A. Let $\hbba$ be the parameters of a model that is initialized at $\hba$ and then trained on task B. Our main result is that if there are more training examples than the intrinsic (latent) dimensionality of the data and if there is not too much noise in the observed features, then
    
\begin{align}
    R(f_{\hbba})-R(f_{\hba}) \lesssim \sqrt{\frac{n}{p}}
\end{align}

with high probability. The result asserts that under our linear model, the extent of catastrophic forgetting is arbitrarily small if the overparameterization ratio, $p/n$, is sufficiently large.  We thus see an analytical illustration that catastrophic forgetting can be ameliorated by overparametization in the case of a suitable linear model.     
The full theorem is stated in Section \ref{theorem} and its  proof is provided in Appendix \ref{lemmas}.

The contributions of this paper are:
\begin{itemize}
    \item We empirically observe that overparameterization can account for a majority of the performance drop due to catastrophic forgetting in a permuted image task using a multi-layer perceptron.
    \item We provide a  linear regression problem that exhibits a corresponding effect for overparameterization and continual learning.
    \item We establish a non-asymptotic bound on the performance drop of this linear model in an orthogonal transformation task setting using results from random matrix theory.  This result provides a formal illustration that continual learning can in some cases be ameliorated by overparameterization.
\end{itemize}

\section{Analysis of Catastrophic Forgetting in a Linear Model} \label{models}

In this section, we present a latent space model for linear regression that we will analyze in order to illustrate that overparameterization can ameliorate catastrophic forgetting.  Our single task model is the latent space model of \cite{hastie2019surprises} without label noise. Then, we present the analogy between this linear model and neural networks.  Next, we empirically demonstrate that under this model, overparameterization ameliorates catastrophic forgetting.  Finally, we present a theorem that establishes that observation with high probability.


\subsection{Latent Space Models for Two Linear Regression tasks} \label{lsm}

Let $\mathcal{Z} = \R^d$, which we call the latent feature space. Consider data for regression generated by a noiseless linear response to standard Gaussian latent features. That is, for some $\theta \in \R^d$, let an example be given by
\begin{align}
    z &\sim \N(0, I_d), \label{data-start} \\
    y &= z^\top\theta.
\end{align}

Let $\mathcal{X} = \R^p$, which we call the observed feature space. We consider the case where, for each example, we have access only to $p$ observed features, given by noisy random projections of the latent features:

\begin{align}
    x &= Wz + u \label{data-end}
\end{align}

where $W \in \R^{p \times d}$ and $u \sim \N(0, I_p)$.  We could take $W$  to have i.i.d. $\N(0, \gam)$ entries, but for mathematical convenience, we will instead study the idealization in which $W$ has columns that form a scaled orthonormal basis of a random $d$-dimensional subspace of $\R^p$. Namely, $W^\top W = p\gam I_d$.  For large $p$, this idealization is approximately satisfied under the above Gaussian model for $W$.

We consider two tasks, $A$ and $B$, both with $n$ examples. Task $A$ has data $(X_A, y) \in \R^{n \times p} \times \R^n$ where each of the $n$ rows of $X_A$ and entries of $y$ are sampled independently by (\ref{data-start}) - (\ref{data-end}). Let $O$ be a random $p \times p$ orthogonal matrix. Task B has data $(X_B, y)$ where $X_B = X_A O^\top$.

We study estimators that are linear in the observed features $x$:

\begin{align}
    f_{\hb} : x \mapsto x^\top\hb,
\end{align}

and we will sometimes refer to the parameters $\hat{\beta}$ as the estimator.
We estimate the parameters of this model by gradient descent with a square loss. We are interested in the case of $d < n < p$. As $n < p$, the solution to this problem depends on initialization and solves the following optimization problem:

\begin{align}
    \arg\min_{\hat{\beta}} \frac{1}{2}\|\hat{\beta} - \beta_0\|^2 \text{ s.t. } y = X\hb, \label{opt-prob}
\end{align}

where $\beta_0$ is the initialization, and $X$ is either $X_A$ or $X_B$, depending on the task being solved. To study the sequential training of tasks A and B, we define the following estimators: 

\begin{itemize}
\item $\hba$ is the solution to task $A$ when initialized at $0$,
\item $\hbb$ is the solution to task $B$ when initialized at $0$,
\item $\hbba$ is the solution to task $B$ when initialized at $\hba$.
\end{itemize}

These parameters are found by solving the following optimization problems:

\begin{align}
    \hba &= \arg\min_{\hat{\beta}} \frac{1}{2}\|\hat{\beta}\|^2 \text{ s.t. } y = X_A\hb, \\
    \hbb &= \arg\min_{\hat{\beta}} \frac{1}{2}\|\hat{\beta}\|^2 \text{ s.t. } y = X_B\hb, \\
    \hbba &= \arg\min_{\hat{\beta}} \frac{1}{2}\|\hat{\beta} - \hba\|^2 \text{ s.t. } y = X_B\hb.
\end{align}

The optimization problem in (\ref{opt-prob}) has the following closed form solution when $X$ has rank $n$:

\begin{align}
    \hat{\beta} &=\beta_0 + X^\top(XX^\top)^{-1}y-X^\top(XX^\top)^{-1}X\beta_0 \\
    &=\beta_0 + X^\top(XX^\top)^{-1}y-\mathcal{P}_{X^\top}\beta_0,
\end{align}

where $\mathcal{P}_{X^\top}$ is the orthogonal projector onto the range of $X^\top$. As $n<p$,  $X_A$ and $X_B$ have rank $n$ with probability $1$, and this gives the following closed forms for $\hba, \hbb, \hbba$:

\begin{align}
    \hba &= X_A^\top(X_AX_A^\top)^{-1}y, \\
    \hbb &= X_B^\top(X_BX_B^\top)^{-1}y, \\
    \hbba &= \hba + \hbb - \mathcal{P}_{X_B^\top}\hba. \label{model-end}
\end{align}

We evaluate these estimators on task A. The risk on task A of an estimator $f$ with parameters $\hb$ is given by


\begin{align}
    R(f_{\hb}) = \s^2 + (\hb-\beta)^\top\Sigma(\hb-\beta) \label{risk-start}
\end{align}

where
\begin{align}
    \Sigma &= WW^\top + I_p \\
    \beta &= (I+WW^\top)^{-1}W\theta \\
    \sigma^2 &= \theta^\top(W^\top W + I_d)^{-1}\theta \label{risk-end}
\end{align}

See Appendix \ref{risk} for the derivation of (\ref{risk-start})--(\ref{risk-end}). It follows from showing that the latent space model described above is equivalent to an anisotropic regression model where $X_A$ has i.i.d. rows ${X_A}_i \sim \N(0, \Sigma)$ and labels $y = X_A\beta+\eps$ where $\eps \sim \N(0, \s^2I_n)$.

\begin{figure}
\begin{center}
    \includegraphics[scale=0.06]{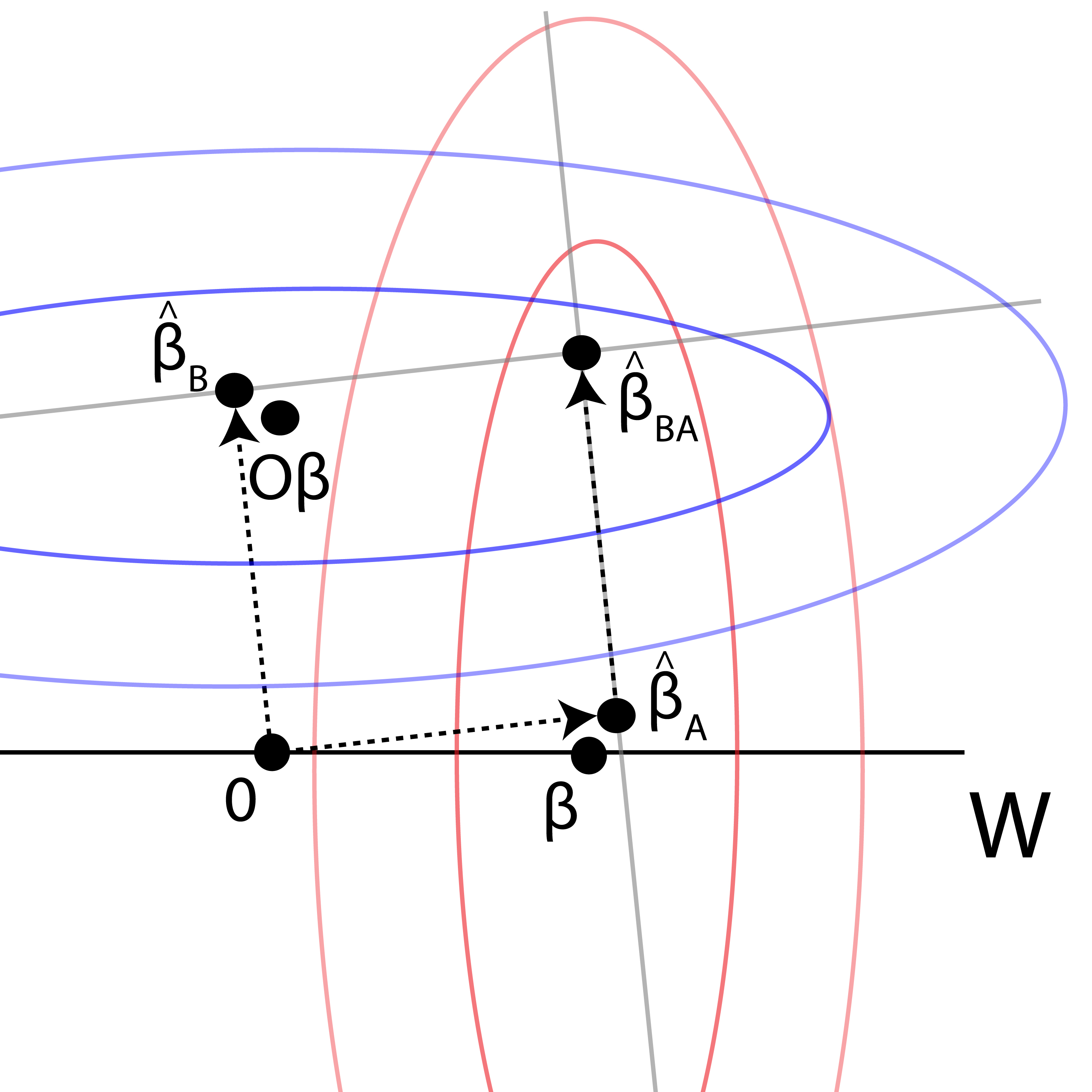}
    \caption{The solid black line depicts the span of $W$. The true parameters corresponding to tasks A and B are given by $\beta \in W$ and $O\beta$.  The gray lines depict the set of solutions to $X_A \beta = y$ and $X_B \beta = y$.  The estimators $\hba, \hbb, \hbba$ are given by orthogonal projections of an initialization on the respective consistent solutions.  The red and blue ellipses depict lines of constant risk for tasks A and B, respectively. 
    }
    \label{manifold-fig}
\end{center}
\end{figure}

We aim to bound $R(f_{\hbba})$ relative to $R(f_{\hba})$. Figure \ref{manifold-fig} illustrates the estimators $\hba, \hbb, \hbba$ and curves of constant risk.  As depicted, if $p$ is large enough, $\hbba$ has low risk on Task A and Task B simultaneously.    

\subsection{Analogy of Linear Model to Neural Networks}

The linear model we study is intended to be a mathematically tractable idealization of a neural network, and it is meant to analytically illustrate that overparameteriation can ameliorate catastrophic forgetting. The analogy of this linear model and neural network training on image data is as follows:

Natural images in a neural network's training distribution can be (approximately) modeled as being on a nonlinear manifold and having a low-dimensional latent representation. Instead of observing the latent representation of an image, the neural network only sees a high-dimensional representation either directly in pixel space or perhaps in a representation computed from pixel space. Either of these representations contain noise in the features used for prediction. Responses can be approximated by a neural network. 

In our linear model, the low-dimensional  representation of an input image is in a $d$-dimensional latent feature linear space.  The responses are linear in the latent features. We assume the response is noiseless for the sake of simplicity, though our results could be extended to the noisy case. In our linear model, predictions are made off of a $p$-dimensional model given by noisy random projections of the latent features. 
We constrain $W$ to have orthonormal columns which is a mathematical idealization of Gaussian measurements. We study two tasks with the same responses like in the permutation task setup, but for mathematical convenience we study tasks that are related by a random orthogonal transformation instead.

\subsection{Numerical Experiment} \label{num-exp}

Before we establish our theoretical result about the system described in Section \ref{lsm}, we provide empirical evidence that the latent space linear regression model above exhibits the phenomenon that overparameterization can ameliorate catastrophic forgetting.  Specifically, we provide empirical evidence that $R(f_{\hbba}) - R(f_{\hba})$ decreases with $p$.

Let $d=20$, $n=100, \gamma=1$, $\beta_0 = \vec{0}$, and $\theta \sim \N(0, I_d)$. We plot $R(f_{\hbba})$, $R(f_{\hba})$, $R(f_{\beta_0})$ as a function of $p \in (n, 2000)$ averaged over 100 samplings of $W, X_A, O, u$.

\begin{figure}[h]
\begin{center}
    \includegraphics[scale=0.5]{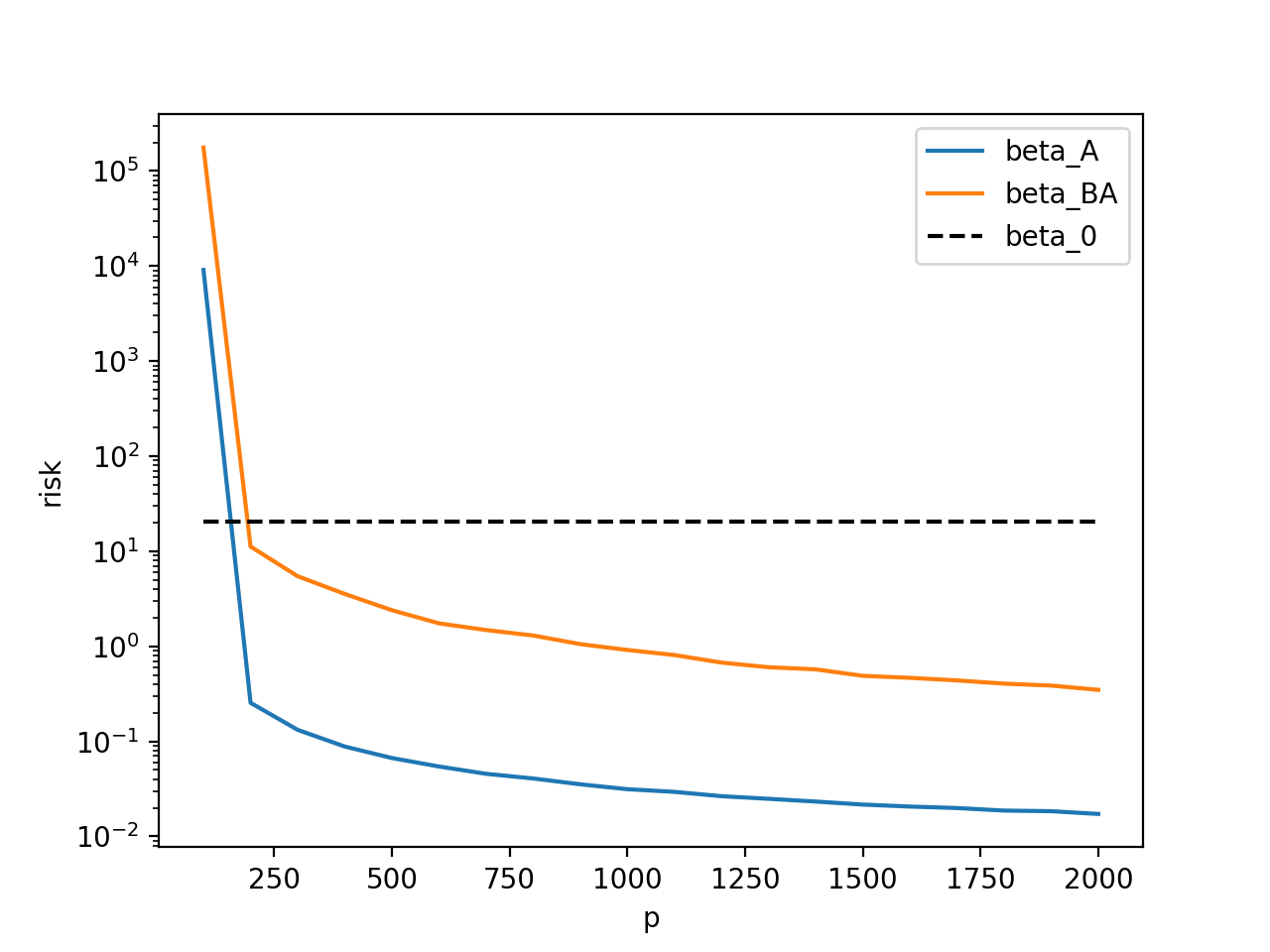}
    \caption{Risk as a function of model complexity $p$ for model (\ref{data-start})--(\ref{model-end}).  The dashed line depics the null risk, corresponding to the zero estimator.  Note the log scale on the vertical axis.  Result of simulated numerical experiment for random orthogonal transformation tasks. Dotted black line denotes risk of null estimator, blue line denotes risk of estimator trained on task A, orange line denotes risk of estimator trained on task A then task B.}
    \label{sim-experiment}
\end{center}
\end{figure}

\noindent
Figure \ref{sim-experiment} shows the results of the experiment.  
We first note that both $\hba$ and $\hbba$ outperform the null risk, given by $\beta=0$.   The null risk, $R(f_{\beta_0})$, defines a baseline that any reasonable model must beat. 
We also observe that $R(f_{\hba})$ and $R(f_{\hbba})$ are  decreasing with $p$ in the overparameterized regime, and that the difference between these risks appears to decrease for increasing $p$. Note the log-scale of the vertical-axis. This provides evidence that catastrophic forgetting is alleviated in the overparameterized regime in our two-task learning setup.  

\subsection{Main Result} \label{theorem}

Our main result is an upper bound on the performance drop, defined as $R(f_{\hba})-R(f_{\hbba})$, for the two-task latent space linear regression model described above and inspired by the double descent literature.  As described in Section \ref{lsm}, we consider (\ref{data-start})--(\ref{model-end}), where $W$ satisfies the following assumption. 


\begin{assumption} \label{body-assumption}
    All non-zero singular values of $W$ are equal. Namely, $W^\top W = p\gamma I_d$.
\end{assumption} 

We begin with a proposition that defines the unlearned baseline for the problem.

\begin{proposition} \label{proposition}
    Fix $\theta \in \R^d$. Let $W \in \R^{p \times d}$ satisfy Assumption \ref{body-assumption}. Then
    
    \begin{align}
        R(f_{0}) = \|\theta\|^2.
    \end{align}
\end{proposition}

This risk calculation agrees with the numerical experiment in Section \ref{num-exp} where $\|\theta\|^2 \approx d = 20$. The result is formally stated and proven in Lemma \ref{null-risk}. For our main result, we prove that if the number of examples exceeds the problem's latent dimensionality, if the number of parameters is sufficiently large relative to the number of examples and relative to the noise level of the observable features, then with high probability, the performance drop is small.

\begin{theorem} \label{main}
    Fix $\theta \in \R^d$. Let tasks $A,B$ be given by (\ref{data-start})--(\ref{model-end}). Let $W \in \R^{p \times d}$ satisfy Assumption \ref{body-assumption} and $n \ge d, p \ge max(17n, 1/\gam)$. Then there exists constant $c>0$ such that with probability at least $1-10e^{-cd}$, the following holds:
    
    \begin{align}
    R(f_{\hbba}) - R(f_{\hba}) \le \left(66\sqrt\frac{n}{p} + \frac{12}{p\gam}\right)\|\theta\|^2
\end{align}

\end{theorem}

Theorem \ref{main} provides an upper bound on the amount of risk gained on task A after subsequential training on task B given by two terms. The first term provides dependence on the overparameterization ratio $p/n$ and decreases as overparameterization becomes more extreme. The second term is given by the signal to noise ratio of the noisy features. This term dominates only when $\gamma \ll 1/\sqrt{np}$. Based on the theorem, we observe that the overparameterization needs only to be linear in $n$ to achieve a negligible performance drop in unregularized sequential task training compared to the baseline of $\|\theta\|^2$ in Proposition \ref{proposition}. This shows that catastrophic forgetting is ameliorated in the overparameterized regime. This result is formally stated in Lemma \ref{perf-drop}. A formal proof and supporting lemmas are supplied in Appendix \ref{lemmas}. We provide a proof sketch here to outline the techniques used.

\subsection{Proof Sketch of Theorem \ref{main}}

For readability, we write $X_A$ as $A$ and $X_B$ as $B$.  As shown in Appendix \ref{risk}, the latent space model described above is equivalent to an anisotropic regression model where $A$ has i.i.d. rows $A_i \sim \N(0, \Sigma)$ and labels $y = A\beta+\eps$ where $\eps \sim \N(0, \s^2I_n)$.

Performance drop is given by 

\begin{align}
    R(f_{\hbba}) - R(f_{\hba}) &= (\hbba-\beta)^\top \Sigma (\hbba-\beta) \notag - (\hba-\beta)^\top \Sigma (\hba-\beta).
\end{align}

After substituting the closed form solutions for $\hba, \hbba$, distributing terms, and applying simple Cauchy-Schwarz and triangle inequalities, we get the following bound:

\begin{align}
    R(f_{\hbba}) - R(f_{\hba}) &\le 8\frac{p\gamma\sqrt{p\gam}}{p\gam+1} \|\theta\|\|\mathcal{P}_W\hbb\| + 14\sqrt{p\gam} \|\theta\|\|\mathcal{P}_{B^\top} \hba\| + 12\frac{p\gam}{(p\gam+1)^2} \|\theta\|^2 \\
    &= I + II + III
\end{align}

Lemmas \ref{b-helper}, \ref{ba-helper} establish results for orthogonal transformations to help bound $\|\mathcal{P}_W\hbb\|, \|\mathcal{P}_{B^\top} \hba\|$ respectively. $\mathcal{P}_W$ is a projection onto a $d$-dimensional space which scales the norm in $I$ by $d/p$. $\mathcal{P}_{B^\top}$ is a projection onto an $n$-dimensional space which scales the norm in $II$ by $n/p$. As $d \le n$ by assumption, $II$ dominates $I$ in the final bound. Using these results and simplifying gives the following bound with probability at least $1-10e^{-cd}$ for constant $c>0$:

\begin{align}
    I+II \le 66\sqrt{\frac{n}{p}}\|\theta\|^2
\end{align}

We directly obtain

\begin{align}
    III \le \frac{12}{p\gam}\|\theta\|^2
\end{align}

Combining these bounds completes the proof.

\section{Discussion} \label{discussion}

Overparameterization is a necessity for continual learning so that there can exist an infinity of potential optima for each task \cite{kirkpatrick2017overcoming}. This makes it likely that there exists an optimum for some task B that is close to the solutions of some task A. We provide experimental evidence that overparameterization can provide additional benefits in combatting catastrophic forgetting for neural networks solving permutation tasks. We use a linear model with clear analogies to neural networks in order to study this behavior theoretically. In our analysis of the linear model in the overparameterized regime, non-asymptotic matrix estimates and results for orthogonal transformations provide bounds on the performance drop. Our main result shows that, under our model, catastrophic forgetting is ameliorated for sufficiently large overparameterization. For the linear setting we study, the behavior we observe can be explained geometrically: overparameterization causes the random orthogonal transformation tasks to live in approximately orthogonal subspaces, so training on subsequent tasks does not interrupt performance on learned tasks.

We view the present work as helping to establish initial results for continual learning theory. Before the field can rigorously understand machine learning algorithms in practice, the behavior of simple systems should be well understood. In particular, the behavior of linear systems with only vanilla SGD is the most natural initial result. Our work remarks that future theory should establish that continual learning algorithms beat not only a moderately parameterized baseline, but also the performance of extremely overparameterized models.

First we address the concern for using permutation tasks as realistic benchmarks for continual learning methods. Researchers believe that permutation tasks only provide a best-case for real world scenarios \cite{farquhar2018towards}. Also, on a number of image classification datasets, MLPs do not experience forgetting when only two permutation tasks are being learned \cite{pfulb2019comprehensive}. Our experiments confirm this effect while also showing that overparameterization mitigates the observable forgetting on 10 task permuted MNIST. Despite these critiques, we use permutation tasks as a launching point for theory because each task is of the same `difficulty' and is amenable to mathematical analysis.

Next we discuss our choice to study the problem with a linear model. Linear regression is the simplest setting, for which we know, that exhibits double descent.  The consensus of several works that study double descent in linear models is that the risk of a model is monotonically decreasing in the overparameterized regime with respect to number of parameters only if the data has low effective dimension and high ambient dimension compared to the number of training samples \cite{dar2021farewell, bartlett2020benign, hastie2019surprises}. In order to have a model that has monotonically decreasing performance drop for a particular continual learning problem, it is a necessity that it exhibits monotonically decreasing risk on a single task. Additionally in recent work, connections have been made between neural networks and linear models using the so-called neural tangent kernel (NTK) phenomenon \cite{jacot2018neural}. The parameterization of a neural network can be so large that training only changes its parameters slightly from its initialization, resulting in functions that can be accurately approximated linearly. Hence it is reasonable that the analysis of linear models can explain the behavior of neural networks.

We now remark at a technical level two choices in our analysis. The first is why we studied the case of random orthogonal transformation tasks instead of permutation tasks. The empirical performance between orthogonal and permutation tasks is similar; they both create tasks that are equally 'difficult' for an MLP to learn, which spares us from needing to quantify problem difficulty. Appendix \ref{perm-exp} provides evidence that permutation and orthogonal transformation tasks have the same difficulty in the linear setting. Also, the mathematical analysis is easier when studying orthogonal transformation tasks. With random orthogonal transformations, any subspace gets mapped to a random subspace, for which the values of coefficients are typically well spread out. With random permutations, some subspaces (e.g. those aligned with the standard basis elements) do not exhibit the same spreading effect, making the technical analysis more involved. Secondly, we do not present a bound on $R(f_{\hba})$, though it is expected to approach zero for large $p$, as suggested by Figure \ref{sim-experiment}. Whether or not this risk goes to 0 in $p$, the performance drop goes to 0 in $p$ while the null estimator remains with constant risk. So the regression problem is being solved arbitrarily well for sufficiently large $p$.

With the growing popularity of continual learning, much of recent work is focused on developing new algorithms to mitigate catastrophic forgetting \cite{kirkpatrick2017overcoming, zenke2017continual, shin2017continual, li2017learning}. Only a few papers study the problem theoretically \cite{knoblauch2020optimal, bennani2020generalisation, doan2021theoretical, heckel2022provable, benzing2022unifying}. \cite{knoblauch2020optimal} uses set theory to prove that, in general, continual learning problems are NP-hard, explaining why generative replay methods perform so well. \cite{bennani2020generalisation} uses the NTK regime to prove generalisation guarantees for an existing continual learning method. \cite{doan2021theoretical} uses an NTK overlap matrix to define a notion of task similarity and show that catastrophic forgetting is more severe when tasks have high similarity. \cite{heckel2022provable} studies a family of continual learning methods that uses approximations of the Hessian to determine parameter importance, presenting scenarios where continual learning provably fails and succeeds. \cite{benzing2022unifying} shows that a number of regularization techniques that seem to be derived from differing philosophies actually all study a variation of the Fisher information matrix. While prior work has presented generalization bounds for existing continual learning techniques, our work illustrates the relationship between overparameterization and generalization.

A natural next step is to study the regimes in which catastrophic forgetting is most problematic. This includes the setting where tasks do not have a nearly orthogonal relationship but also when data does not necessarily live on a low-dimensional manifold. We are also interested in understanding how the ideas of this paper generalize to other continual learning benchmarks and for more general neural network architectures. Prior work found experimental evidence that catastrophic forgetting is most severe not when tasks are very dissimilar but when they only have an intermediate level of similarity \cite{ramasesh2020anatomy}. Using orthogonality as a proxy for task similarity, this agrees with our work that shows that nearly orthogonal tasks are less prone to catastrophic forgetting. An interesting future work would be to formalize this notion of task similarity for our model. Moving forward, one goal of theory in continual learning is to be able to analytically compare algorithms. Our work provides a foundation of understanding this behavior in a simple linear regression setting. In order to push this work forward, either non-linear models need to be studied or tasks that are related by something more complex than permutations.

\bibliography{main}{}
\bibliographystyle{plain}

\newpage
\appendix
\onecolumn

\section{Description of Continual Learning Techniques} \label{cl-techniques}

Synaptic Intelligence (SI) is a regularization technique that assigns to each parameter of the network an estimate of importance for learned tasks \cite{zenke2017continual}. This weight is determined in an online manner by tracking the amount that each parameter contributed to the decrease in loss during training. The weight is then used to penalize changes to the network parameters during subsequent training in the form of a regularization term added to the loss function.

Elastic Weight Consolidation (EWC) is a regularization technique that determines the importance of network weights using an estimation of the Fisher Information Matrix \cite{kirkpatrick2017overcoming}. Near a minimum of the loss function, the diagonals of the Fisher matrix act as an estimate of the second order derivative of the loss with respect to each parameter. The magnitude of this derivative is used as a proxy for how sensitive the loss function is to fluctuation of the parameter. Constraining parameters according to their corresponding Fisher diagonal entries shows as an effective way of retaining the values of important weights from previous tasks while training on new ones.

\section{MNIST Experiments} \label{mnist-appendix}

\begin{table}
\begin{center}
\caption{Hyperparameters for the MNIST experiments}
\begin{tabular}{ |c|c|c|c| } 
\hline
Hyperparameter & SI & EWC & SGD \\
\hline
learning rate & 0.1 & 0.01 & 0.1 \\ 
dropout & 0 & 0 & 0.5 \\ 
batch size & 64 & 128 & 64 \\
epochs / dataset & 5 & 20 & 5 \\ 
c & 0.1 & & \\
$\xi$ & 0.1 & & \\
$\lambda$ & & 150 & \\
fisher sample size & & 1,000 & \\
\hline
\end{tabular}
\label{hyperparameters}
\end{center}
\end{table}

Table \ref{hyperparameters} reports the hyperparameters used in the MNIST experiments. We adopted the same hyperparameters for SI as in the original paper \cite{zenke2017continual}. To our surprise, EWC with default hyperparameters \cite{kirkpatrick2017overcoming} did not compete with SI. A basic grid search gave us a model that was more competitive. Blank entries mean that the hyperparameter is not relevant for the particular method. Curves for $w=7,9$ are omitted due to computational constraints in computing Fisher matrix estimates.

\section{Permuted Numerical Experiment} \label{perm-exp}

\begin{figure}[h]
\begin{center}
    \includegraphics[scale=0.5]{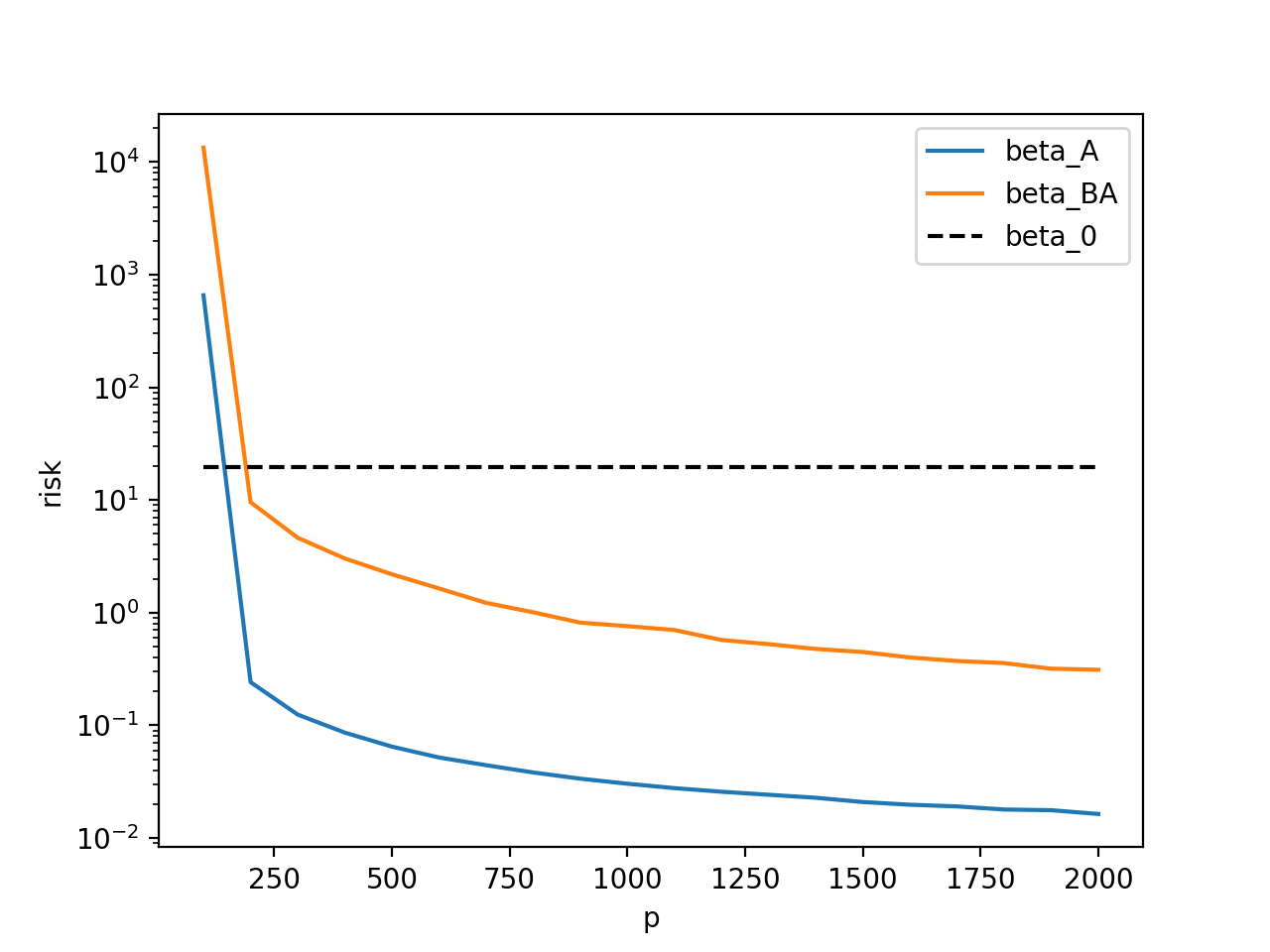}
    \caption{Result of simulated numerical experiment for random permutation tasks. Dotted black line denotes risk of null estimator, blue line denotes risk of estimator trained on task A, orange line denotes risk of estimator trained on task A then task B.}
    \label{sim-experiment-perm}
\end{center}
\end{figure}

\noindent
Figure \ref{sim-experiment-perm} shows the result of the numerical experiment in Figure \ref{sim-experiment} but with a random permutation matrix instead of a random orthogonal matrix. We observe that the same behavior holds in this scenario.

\section{Equivalence of Models and Derivation of Risk} \label{risk}


Recall in Section \ref{lsm} where we defined the LSM model for linear regression. In this section we show that LSM is equivalent to an anisotropic regression model (ARM). We then use ARM to define the risk expression that we analyze theoretically.

We begin by defining ARM. Define data matrix $X \in \R^{n \times p}$ and responses $y=X\beta+\eps$ where $\eps \sim \N(0, \sigma^2I_n)$,  $\sigma^2 = \theta^\top(W^\top W + I_d)^{-1}\theta$, and $\beta = (I+WW^\top)^{-1}W\theta$ for some $\theta \in \R^d, W \in \R^{p \times d}$. Let rows $X_i$ be independent random vectors in $\R^p$ with covariance $\Sigma = WW^\top+I_p$. Then the model is defined by the distribution over $(X,y)$.

Next we show that ARM is equivalent to LSM. First observe that for both models, $(y_i, x_i^\top) \in \R^{p+1}$ are centered Gaussian vectors. Thus to show that they induce the same distribution, it suffices to show that they have the same covariance.

\begin{align}
    Cov((y_i, x_i^\top)^\top) = \E (y_i, x_i^\top)^\top (y_i, x_i^\top) = \E\begin{bmatrix}
y_i^2 & y_ix_i^\top \\
y_ix_i & x_ix_i^\top \label{cov}
\end{bmatrix}
\end{align}

We then compute the covariance matrices for each model.

\textbf{Under LSM, we have:}

\begin{align}
    \E[y_i^2] &= \E(\theta^\top z_i)(z_i^\top\theta) = \theta^\top I \theta \\
    \E[y_ix_i] &= \E(Wz_i + u_i)(z_i^\top\theta) = W\theta \\
    \E[x_ix_i^\top] &= \E(Wx_i+u_i)(z_i^\top W^\top + u_i) = WW^\top + I
\end{align}

Plugging these quantities into (\ref{cov}) gives:

\begin{align}
    Cov((y_i, x_i^\top)^\top) = \begin{bmatrix}
\|\theta\|^2 & (W\theta)^\top \\
W\theta & I+WW^\top
\end{bmatrix}
\end{align}

\textbf{Under ARM, we have:}

\begin{align}
    \E[y_i^2] &= \E(\beta^\top x_i + \eps_i)(x_i^\top\beta + \eps_i) = \beta^\top(I+WW^\top)\beta + \s^2 \\
    \E[y_ix_i] &= \E(x_i(x_i^\top\beta + \eps_i)) = \E (x_ix_i^\top\beta + x_i\eps_i) = (I+WW^\top)\beta \\
    \E[x_ix_i^\top] &= I + WW^\top
\end{align}

Plugging these quantities into (\ref{cov}) gives:

\begin{align}
    Cov((y_i, x_i^\top)^\top) = \begin{bmatrix}
\beta^\top(I+WW^\top)\beta + \s^2 & ((I+WW^\top)\beta)^\top \\
(I+WW^\top)\beta & I+WW^\top
\end{bmatrix}
\end{align}


\textbf{We now show equivalence of the covariance matrices.} Recall that for ARM, \\$\beta = (I+WW^\top)^{-1}W\theta$ and $\s^2 = \theta^\top(I+W^\top W)^{-1}\theta$. We first show equivalence of the first row, first column entries of the covariance matrices:

\begin{align}
    \beta^\top(I+WW^\top)\beta + \s^2 &= \theta^\top W^\top (I+WW^\top)^{-1}W\theta + \theta^\top(I+W^\top W)^{-1}\theta \\
    &= \theta^\top (W^\top (I+WW^\top)^{-1}W + (I+W^\top W)^{-1})\theta
\end{align}

By Lemma \ref{w-lem}, $(I+WW^\top)^{-1}W = W(I+W^\top W)^{-1}$. This gives

\begin{align}
    \beta^\top(I+WW^\top)\beta + \s^2 &= \theta^\top (W^\top W(I+W^\top W)^{-1} + (I+W^\top W)^{-1})\theta \\
    &= \theta^\top (W^\top W + I)(I+W^\top W)^{-1}\theta \\
    &= \theta^\top\theta = \|\theta\|^2
\end{align}

Next we show equivalence of the second row, first column entries of the covariance matrices:

\begin{align}
    (I+WW^\top)\beta &= (I+WW^\top)(I+WW^\top)^{-1}W\theta = W\theta
\end{align}

The equivalence of the first row, second column entries also follows from this equality. The equivalence of the second row, second column entries is trivial.

Finally we derive the expression for the risk of ARM. By definition, the risk of an estimator $f$ with parameters $\hb$ has the following form:

\begin{align}
    R(f_{\hb}) &= \E_{x,y}\|\hb^\top x - y\|^2 \\
    &= \E_{x,\eps}\|\hb^\top x - \beta^\top x - \eps\|^2 \\
    &= \E_{x}\|\hb^\top x - \beta^\top x\|^2 + \E_{\eps}\|\eps\|^2 \\
    &= \E_{x}\|(\hb - \beta)^\top x\|^2 + \E_{\eps}\|\eps\|^2 \\
    &= \E_{x}(\hb - \beta)^\top xx^\top (\hb - \beta) + \E_{\eps}\|\eps\|^2 \\
    &= (\hb - \beta)^\top \Sigma (\hb - \beta) + \sigma^2
\end{align}

where the third equality holds from independence of $\epsilon$ and the sixth equality holds by definition of covariance. 

We choose to study ARM with a slightly different but equivalent expression for $\beta$. Using Lemma \ref{w-lem}, $\beta = (I+WW^\top)^{-1}W\theta = W(W^\top W + I)^{-1}\theta$.

\section{Supporting Lemmas} \label{lemmas}

We begin with an assumption, inspired by \cite{hastie2019surprises}, that all non-zero singular values of $W$ are equal.

\begin{assumption} \label{assumption}
    All non-zero singular values of $W$ are equal. Namely, $W^\top W = p\gamma I_d$.
\end{assumption} 

\begin{lemma} \label{s-approx}
    Assume $W \in \R^{p \times d}$ satisfies Assumption \ref{assumption}. Then
    
    \begin{align}
        WW^\top = p\gamma\mathcal{P}_W
    \end{align}
    
    where $\mathcal{P}_W$ is the orthogonal projection onto the range of $W$.
\end{lemma}

\begin{proof} We have that
    
    \begin{align}
        WW^\top
        = WW^\top\frac{p\gam}{p\gam} = p\gam W\left(\frac{1}{p\gam}I_p\right)W^\top
    \end{align}
   
\noindent
By Assumption \ref{assumption}, we have

    \begin{align}
        WW^\top 
        = p\gam W(W^\top W)^{-1}W^\top
    \end{align}
    
\noindent
$W^\top W$ has full rank with probability 1, so $W(W^\top W)^{-1}W^\top$ is given explicitly by $\mathcal{P}_W$, which completes the proof. \end{proof}

\begin{lemma} \label{null-risk}
    Let $\Sigma = WW^\top+I_p$ where $W \in \R^{p \times d}$ satisfies Assumption \ref{assumption}. For some $\theta \in \R^d$, let $\beta = W(W^\top W + I_d)^{-1}\theta$. Then
    
    \begin{align}
        R(f_{\vec{0}}) = \|\theta\|^2
    \end{align}
\end{lemma}

\begin{proof} We have that

\begin{align}
    R(f_{\vec{0}}) &= 
    (\vec{0}-\beta)^\top \Sigma (\vec{0}-\beta) + \s^2 \\
    &= \beta^\top \Sigma \beta + \s^2
\end{align}
    
\noindent
By Lemma \ref{s-approx}, $\Sigma = p\gamma\mathcal{P}_W+I_p$ where $\mathcal{P}_W$ is the orthogonal projection onto the range of $W$, which gives
    
\begin{align}
    R(f_{\vec{0}}) &= \beta^\top(p\gam \mathcal{P}_W + I_p)\beta + \s^2
\end{align}

\noindent
Since $\beta \in range(W)$,

\begin{align}
    R(f_{\vec{0}}) &=  (p\gamma + 1)\|\beta\|^2 + \s^2 \\
    &=  (p\gamma + 1)\theta^\top(W^\top W + I_d)^{-1}W^\top W(W^\top W + I_d)^{-1}\theta  + \theta^\top(W^\top W+I_d)^{-1}\theta \\
    &= (p\gamma + 1)\theta^\top((p\gamma+1)I_d)^{-1} p\gamma I_d ((p\gamma+1)I_d)^{-1}\theta  + \theta^\top((p\gamma+1)I_d)^{-1}\theta \\
    &= \left(\frac{p\gamma}{p\gamma+1} + \frac{1}{p\gamma+1}\right)\|\theta\|^2 = \|\theta\|^2
\end{align} \end{proof}

\begin{lemma}
\label{lem:conc1}
Let $x \sim \mathcal{N}(0,I_d)$, and $\epsilon \leq 1$, then
\[
\mathbb{P}\left ( d(1-\epsilon) \leq \|x\|_2^2 \leq d(1+\epsilon) \right) \geq 1- e^{-c \epsilon^2 d}
\]
where $c>0$ is an absolute constant.
\end{lemma}
\begin{proof}
This statement follows from Corollary 5.17 in \cite{vershynin2010introduction}, concerning concentration of sub-exponential random variables.
\end{proof}

\begin{lemma} \label{b-helper}
    Assume $W \in \R^{p \times d}$ satisfies Assumption \ref{assumption}. Let $O$ be a random $p \times p$ orthogonal matrix. Fix $v \in \R^{p \times p}$. Then, with probability at least $1 - 2 e^{-c_1 d}$,
    \begin{align}
        \|\mathcal{P}_W O v\|^2 \le \frac{2d}{p} \|v\|^2,
    \end{align}
    for some universal constant $c_1>0$.
\end{lemma}

\begin{proof}

\noindent
Let $x=Ov$, and note that $\|x\| = \|v\|$, $\|x\|>0$ with probability 1 and $\frac{x}{\|x\|} \sim \text{Uniform}(\mathbb{S}^{p-1})$.
Letting $z \sim \mathcal{N}(0, I_p)$, we have that
\begin{align}
\|\mathcal{P}_W O \mathcal{P}_{A^\top} v\| &\leq \Bigl\| \mathcal{P}_W \frac{x}{\|x\|} \Bigr\| \|v\| \stackrel{d}{=} \Bigl\| \mathcal{P}_W \frac{z}{\|z\|} \Bigr\| \|v\|
\end{align}
where the symbol $\stackrel{d}{=}$ means equality in distribution.  
Applying Lemma~\ref{lem:conc1} twice, we get that for any $\epsilon < 1$, with probability at least $1 - e^{-c \epsilon^2 p} - e^{-c \epsilon^2 d}$, 
\begin{align}
    \| \mathcal{P}_W \frac{z}{\|z\|} \|\|v\|  \leq \frac{\|\mathcal{P}_W z\|}{\sqrt{p}\sqrt{1-\epsilon}}\|v\| \leq \frac{\sqrt{d}\sqrt{1+\epsilon}}{\sqrt{p}\sqrt{1-\epsilon}}\|v\|
\end{align}
for some universal constant $c>0$.  By choosing suitable $\epsilon$, we obtain that for $c_1 = c \epsilon^2$, with probability at least $1-2 e^{-c_1 d}$, $\|\mathcal{P}_W O  v\|^2 \leq  \frac{2d}{p}\|v\|^2$. \end{proof}

\begin{lemma} \label{ba-helper}
    Define $A \in \R^{n \times p}$ with rows $A_i$ as independent random vectors in $\R^p$ with covariance $\Sigma = WW^\top + I_p$ where $W \in \R^{p \times d}$ satisfies Assumption \ref{assumption}. Let $O$ be a random $p \times p$ orthogonal matrix. Fix $v \in range(A^\top)$. Then with probability at least $1-2e^{-c_1 n}$
    \begin{align}
        \|\mathcal{P}_{OA^\top} v\|^2 \le \frac{2n}{p} \|v\|^2
    \end{align}
\end{lemma}

for some universal constant $c_1>0$.

\begin{proof}

We have that 
\begin{align}
    \|\mathcal{P}_{OA^\top} \mathcal{P}_{A^\top} v\| &= \|O \mathcal{P}_{A^\top} O^\top \mathcal{P}_{A^\top} v\| = \| \mathcal{P}_{A^\top} O^\top \mathcal{P}_{A^\top} v\|  \\
    &= \| \mathcal{P}_{A^\top} \frac{O^\top \mathcal{P}_{A^\top} v}{\| O^\top \mathcal{P}_{A^\top} v \| }\| \| O^\top \mathcal{P}_{A^\top} v \| \\
    &\leq \| \mathcal{P}_{A^\top} \frac{O^\top \mathcal{P}_{A^\top} v}{\| O^\top \mathcal{P}_{A^\top} v \| }\|\|v\| \\
    &\stackrel{d}{=} \| \mathcal{P}_{A^\top} \frac{z}{\|z\|} \| \|v\|
\end{align}
where $z \sim \mathcal{N}(0, I_p)$ and the last equality follows from the rotational invariance of $O$. Applying Lemma \ref{lem:conc1} twice, we get that for any $\epsilon < 1$, with probability at least $1 - e^{-c \epsilon^2 p} - e^{-c \epsilon^2 n}$,

\begin{align}
    \| \mathcal{P}_{A^\top} \frac{z}{\|z\|} \|\|v\| \le \frac{\|\mathcal{P}_{A^\top} z\|}{\sqrt{p}\sqrt{1-\epsilon}}\|v\| \le \frac{\sqrt{n}\sqrt{1+\epsilon}}{\sqrt{p}\sqrt{1-\epsilon}}\|v\|
\end{align}

\noindent
for some universal constant $c>0$. By choosing suitable $\epsilon$, we obtain that for $c_1=c\epsilon^2$, with probability at least $1-2e^{-c_1 n}$, $\|\mathcal{P}_{OA^\top} v\|^2 \le \frac{2n}{p}\|v\|^2$. \end{proof}

\begin{lemma} \label{a-row-bound}

Let $a \in \R^p$ be generated by $\N(0,\Sigma)$, $\Sigma=WW^\top+I_p$ where $W \in \R^{p \times d}$ satisfies Assumption \ref{assumption}. Then
    
    \begin{align}
        \E\|a\|_2^2 = dp\gam+p
    \end{align}
\end{lemma}

\begin{proof}

It holds that $\E\|a\|_2^2 = \|\Sigma^{1/2}\|_F^2$ \cite{vershynin2010introduction}.

\begin{align}
    \|\Sigma^{1/2}\|_F^2 &= \|\Sigma\|_*
\end{align}

\noindent
where $\|\|_*$ denotes the nuclear norm. By Lemma \ref{s-approx}, $\Sigma = p\gam\mathcal{P}_W+I_p$, which has $d$ singular values of $p\gam+1$ and $p-d$ singular values of 1. So

\begin{align}
    \|\Sigma\|_* = d(p\gam+1) + p-d
    = dp\gam+p
\end{align}

\end{proof}

\begin{lemma} \label{A-min}
    Define $A \in \R^{n \times p}$ with rows $A_i$ independent random vectors in $\R^p$ with covariance $\Sigma = p\gam\mathcal{P}_W+I_p$ where $W \in \R^{p \times d}$ satisfies Assumption \ref{assumption}. Then with probability at least $1-2e^{-n}$,
    
    \begin{align}
        \sigma_{min}(A)^2 \ge (\sqrt{p-d}-2\sqrt{n})^2
    \end{align}
\end{lemma}

\begin{proof} WLOG let $range(W) = span(e_1, ..., e_d)$. Then we can decompose $A$ into two pieces: $A_{(1)} \in \R^{n \times d}$ with i.i.d. $\N(0, p\gam+1)$ entries and $A_{(2)} \in \R^{n \times p-d}$ with i.i.d. $\N(0, 1)$ entries. This gives

\begin{align}
    \s_{min}(A)^2 &= \s_{min}(AA^\top) = \s_{min}(A_{(1)}A_{(1)}^\top + A_{(2)}A_{(2)}^\top) \\ &\ge \s_{min}(A_{(2)}A_{(2)}^\top) = \s_{min}(A_{(2)})^2 = \s_{min}(A_{(2)}^\top)^2
\end{align}

\noindent
By Theorem 5.39 in \cite{vershynin2010introduction}, $\s_{min}(A_{(2)}^\top)^2 \ge (\sqrt{p-d} - 2\sqrt{n})^2$ with probability at least $1-2e^{-n}$. \end{proof}

\begin{lemma} \label{pd-helper}
    Define $A \in \R^{n \times p}$ with rows $A_i$ independent random vectors in $\R^p$ with covariance $\Sigma = WW^\top+I_p$ where $W \in \R^{p \times d}$ satisfies Assumption \ref{assumption}. Let $\eps \sim \N(0, \s^2I_n)$ and $\sigma^2 = \theta^\top(W^\top W + I_d)^{-1}\theta$ for some $\theta \in \R^d$. Then with probability at least $1-2e^{-n}$,
    
    \begin{align}
        \|A^\dagger\eps\|^2 \le \frac{n\|\theta\|^2}{p\gam(\sqrt{p-d}-2\sqrt{n})^2},
    \end{align}
    
    where $A^\dagger$ is the pseudoinverse of $A$.
    
    \begin{proof} We have that
        
        \begin{align}
            \|A^\dagger\eps\|^2 \le \|A^\dagger\|^2\|\eps\|^2
        \end{align}
        
    \noindent
    We have that $\|\eps\|^2 = n\s^2 = n\theta^\top(W^\top W+I_d)^{-1}\theta$. Under Assumption \ref{assumption}, this gives $\|\eps\|^2 = \frac{n\|\theta\|^2}{p\gam+1}$,
    
    \begin{align}
        \|A^\dagger\eps\|^2 \le \|A^\dagger\|^2\frac{n\|\theta\|^2}{p\gam+1}
    \end{align}
    
    \noindent
    It holds that $\|A^\dagger\|=1/\s_{min}(A^\top)$ where $\s_{min}(A^\top)$ is the smallest singular value of $A^\top$. By Lemma \ref{A-min}, $\s_{min}(A) \ge \sqrt{p-d}-2\sqrt{n}$ with probability at least $1-2e^{-n}$. This gives
    
    \begin{align}
        \|A^\dagger\eps\|^2 \le \frac{1}{(\sqrt{p-d}-2\sqrt{n})^2} \cdot \frac{n\|\theta\|^2}{p\gam+1} \le \frac{n\|\theta\|^2}{p\gam(\sqrt{p-d}-2\sqrt{n})^2}
    \end{align}
        
    \end{proof}
\end{lemma}

\begin{lemma} \label{b-bound}
    Suppose $W \in \R^{p \times d}$ satisfies Assumption \ref{assumption}. Let $\beta = W(W^\top W+I_d)^{-1}\theta$ for some $\theta \in \R^d$. If $p \ge 1/\gam$ and $p \ge 16n+d$, then
    
    \begin{align}
        \frac{\sqrt{n}\|\theta\|}{\sqrt{p\gam}(\sqrt{p-d}-2\sqrt{n})} \le \|\beta\|
    \end{align}
    
\end{lemma}

\begin{proof}
    We have that $\|\beta\|^2 = \theta^\top(W^\top W+I_d)^{-1}W^\top W(W^\top W+I_d)\theta$. Using Assumption \ref{assumption}, this gives $\|\beta\| = \frac{\sqrt{p\gam}}{p\gam+1}\|\theta\|$. Suppose $p \ge 1/\gamma$, then we have that $\|\beta\| \ge \frac{1}{2\sqrt{p\gam}}\|\theta\|$. When $p \ge 16n+d$, $\frac{\sqrt{n}}{\sqrt{p-d}-2\sqrt{n}} \le \frac{1}{2}$, which gives
    
    \begin{align}
        \frac{\sqrt{n}\|\theta\|}{\sqrt{p\gam}(\sqrt{p-d}-2\sqrt{n})} \le \frac{1}{2\sqrt{p\gam}}\|\theta\| \le \|\beta\|
    \end{align} 
\end{proof}

\begin{theorem} \label{perf-drop}
    Define data matrix $A \in \R^{n \times p}$ and responses $y=A\beta+\eps$ where $\eps \sim \N(0, \sigma^2I_n)$,  $\sigma^2 = \theta^\top(W^\top W + I_d)^{-1}\theta$, and $\beta = W(W^\top W + I_d)^{-1}\theta$ for some $\theta \in \R^d$. Let rows $A_i$ be independent random vectors in $\R^p$ with covariance $\Sigma = WW^\top+I_p$ where $W \in \R^{p \times d}$ follows Assumption \ref{assumption} and $n \ge d, p \ge max(17n, 1/\gam)$. Let $O$ be a random $p \times p$ orthogonal matrix and $B = AO^\top$. Let $\hba$ be the parameters of the minimum norm estimator on $A$, and $\hbba$ be the parameters of the estimator on $B$ using $\hba$ as initialization. Let $R(f_{\hb})$ be the risk on task A of an estimator with parameters $\hb$. Then there exists constant $c>0$ such that with probability at least $1-10e^{-cd}$, the following holds:
    
    \begin{align}
    R(f_{\hbba}) - R(f_{\hba}) \le \left(\frac{ 66\sqrt{n}}{\sqrt{p}} + \frac{12}{p\gam}\right)\|\theta\|^2
\end{align}
    
\end{theorem}

\begin{proof} We have that

\begin{align}
    R(f_{\hbba}) - R(f_{\hba}) &= (\hbba-\beta)^\top \Sigma (\hbba-\beta) - (\hba-\beta)^\top \Sigma (\hba-\beta) \\
    &= (\hba + \hbb - \mathcal{P}_{B^\top} \hba-\beta)^\top \Sigma (\hba + \hbb - \mathcal{P}_{B^\top} \hba-\beta)\nonumber\\ &\qquad- (\hba-\beta)^\top \Sigma (\hba-\beta)
\end{align}

Distributing terms with $\hbb$ and $\pba$ gives

\begin{align}
    R(f_{\hbba}) - R(f_{\hba}) &= (\hba-\beta)^\top \Sigma (\hba-\beta) - (\hba-\beta)^\top \Sigma (\hba-\beta) + 2\hba^\top\Sigma\hbb - 2\hbb^\top\Sigma\beta \nonumber\\
    &\qquad + \hbb^\top\Sigma\hbb-2\hba^\top\Sigma\pba - 2\hbb^\top\Sigma\pba+2\beta^\top\Sigma\pba \nonumber\\
    &\qquad + (\pba)^\top\Sigma\pba \\
    &= 2\hba^\top\Sigma\hbb - 2\hbb^\top\Sigma\beta + \hbb^\top\Sigma\hbb-2\hba^\top\Sigma\pba - 2\hbb^\top\Sigma\pba\nonumber\\
    &\qquad+2\beta^\top\Sigma\pba + (\pba)^\top\Sigma\pba
\end{align}

By Lemma \ref{s-approx}, $\Sigma = p\gamma\mathcal{P}_W+I_p$ where $\mathcal{P}_W$ is the orthogonal projection onto the range of $W$. This implies that $\|\Sigma\|=p\gam+1$, giving

\begin{align}
    R(f_{\hbba}) - R(f_{\hba}) &= 2\hba^\top(p\gamma \mathcal{P}_W + I_p)\hbb - 2\hbb^\top(p\gamma \mathcal{P}_W + I_p)\beta + \hbb^\top(p\gamma \mathcal{P}_W + I_p)\hbb \nonumber\\&\qquad + (\pba)^\top(p\gamma \mathcal{P}_W + I_p)(\pba-2\hba-2\hbb+2\beta) \\
    &\le 2p\gamma\hba^\top\mathcal{P}_W\hbb + 2\hba^\top\hbb - 2p\gamma\hbb^\top\mathcal{P}_W\beta - 2\hbb^\top\beta + p\gamma\hbb^\top \mathcal{P}_W\hbb + \hbb^\top\hbb \nonumber\\&\qquad + (p\gam+1)\|\pba\|\|\pba-2\hba-2\hbb+2\beta\|
\end{align}

\noindent
Applying Cauchy-Schwarz and triangle inequality gives the following bound:
    
\begin{align}
   R(f_{\hbba}) - R(f_{\hba}) &\le 2p\gamma\|\hba\|\|\mathcal{P}_W\hbb\| + 2\|\hba\|\|\hbb\| + 2p\gamma\|\beta\|\|\mathcal{P}_W\hbb\| + 2\|\hbb\|\|\beta\| \nonumber\\
    &\qquad + p\gamma\|\hbb\|\| \mathcal{P}_W\hbb\| + \|\hbb\|^2 \nonumber\\
    &\qquad + (p\gamma+1)\|\mathcal{P}_{B^\top} \hba\|(\|\pba\|+2\|\hba\|+2\|\hbb\|+2\|\beta\|)
\end{align}

\noindent
By definition in Section \ref{models}, $\hba = \mathcal{P}_{A^\top}\beta + A^\dagger \eps$ and $\hbb = O\mathcal{P}_{A^\top}\beta + OA^\dagger \eps$ and it holds that $\|\mathcal{P}_{A^\top}\beta\| \le \|\beta\|$. By Lemma \ref{pd-helper}, $\|A^\top(AA^\top)^{-1}\eps\| \le \frac{\sqrt{n}\|\theta\|}{\sqrt{p\gam}(\sqrt{p-d}-2\sqrt{n})}$ with probability at least $1-2e^{-n}$ (call this Event E). So by Lemma \ref{b-bound} if $p \ge max(17n, 1/\gam)$, then $\|\hba\| \le 2\|\beta\|$ and $\|\hbb\| \le 2\|\beta\|$, which gives

\begin{align}
    R(f_{\hbba}) - R(f_{\hba}) &\le 8p\gamma\|\beta\|\|\mathcal{P}_W\hbb\| + 14(p\gam+1)\|\beta\|\|\mathcal{P}_{B^\top} \hba\| + 12\|\beta\|^2
\end{align}

\noindent
We have that $\|\beta\|^2 = \theta^\top(W^\top W+I_d)^{-1}W^\top W(W^\top W+I_d)\theta$. Using Assumption \ref{assumption}, this gives $\|\beta\| = \frac{\sqrt{p\gam}}{p\gam+1}\|\theta\|$,

\begin{align}
    R(f_{\hbba}) - R(f_{\hba}) &\le 8\frac{p\gamma\sqrt{p\gam}}{p\gam+1} \|\theta\|\|\mathcal{P}_W\hbb\| + 14\sqrt{p\gam} \|\theta\|\|\mathcal{P}_{B^\top} \hba\| + 12\frac{p\gam}{(p\gam+1)^2} \|\theta\|^2 \\
    &= I + II + III
\end{align}

\textbf{We will bound each of these terms separately, starting with $I$}:

Substituting $\hbb=B^\top(BB^\top)^{-1}y$ into this expression and distributing accordingly, we get that $\mathcal{P}_W\hbb = \mathcal{P}_WO\mathcal{P}_{A^\top}\beta + \mathcal{P}_WOA^\top(AA^\top)^{-1}\eps$,

\begin{align}
    I 
    &= 8\frac{p\gamma\sqrt{p\gam}}{p\gam+1} \|\theta\|\|\mathcal{P}_WO\mathcal{P}_{A^\top}\beta + \mathcal{P}_WOA^\top(AA^\top)^{-1}\eps\| \\
    &\le 8\frac{p\gamma\sqrt{p\gam}}{p\gam+1} \|\theta\|\|\mathcal{P}_WO\mathcal{P}_{A^\top}\beta\| + 8\frac{p\gamma\sqrt{p\gam}}{p\gam+1} \|\theta\|\| \mathcal{P}_WOA^\top(AA^\top)^{-1}\eps\|
\end{align}

\noindent
By Lemma \ref{b-helper}, there exists constant $c_1>0$ such that $\|\mathcal{P}_WO\mathcal{P}_{A^\top}\beta\| \le 1.5\sqrt{\frac{d}{p}}\|\beta\|$ and $\| \mathcal{P}_WOA^\top(AA^\top)^{-1}\eps\| \le 1.5\sqrt{\frac{d}{p}}\|A^\top(AA^\top)^{-1}\eps\|$ with probability at least $1-2e^{-c_1d}$ each. By Lemma \ref{pd-helper}, $\|A^\top(AA^\top)^{-1}\eps\| \le \frac{\sqrt{n}\|\theta\|}{\sqrt{p\gam}(\sqrt{p-d}-2\sqrt{n})}$ (failure probability already accounted for on Event E). This gives the following bound with probability at least $1-4e^{-c_1d}$:

\begin{align}
    I &\le 12\frac{p\gam\sqrt{d\gam}}{p\gam+1}\|\theta\|\|\beta\| + 12\frac{\gamma\sqrt{ndp}}{(p\gam+1)(\sqrt{p-d}-2\sqrt{n})} \|\theta\|^2
\end{align}

\noindent
Substituting $\|\beta\| = \frac{\sqrt{p\gam}}{p\gam+1}\|\theta\|$ gives
    
\begin{align}
    I \le 12\frac{p\gam^2\sqrt{dp}}{(p\gam+1)^2}\|\theta\|^2 + 12\frac{\gamma\sqrt{ndp}}{(p\gam+1)(\sqrt{p-d}-2\sqrt{n})} \|\theta\|^2
\end{align}

\noindent
Using $p\gam+1 > p\gam$ gives

\begin{align}
    I &\le 12\frac{\sqrt{d}}{\sqrt{p}}\|\theta\|^2 + 12\frac{\sqrt{nd}}{\sqrt{p}(\sqrt{p-d}-2\sqrt{n})} \|\theta\|^2
\end{align}

\noindent
\textbf{Now we bound term II}:

Substituting $\hba=A^\top(AA^\top)^{-1}y$ into this expression and distributing accordingly, we get that $\mathcal{P}_{B^\top} \hba = \mathcal{P}_{B^\top} \mathcal{P}_{A^\top} \beta + \mathcal{P}_{B^\top}A^\top(AA^\top)^{-1}\eps$. This gives the following bound:

\begin{align}
    II &= 14\sqrt{p\gam} \|\theta\|\|\mathcal{P}_{B^\top} \mathcal{P}_{A^\top} \beta + \mathcal{P}_{B^\top}A^\top(AA^\top)^{-1}\eps\| \\
    &\le 14\sqrt{p\gam} \|\theta\|\|\mathcal{P}_{B^\top} \mathcal{P}_{A^\top} \beta\| + 14\sqrt{p\gam} \|\theta\|\| \mathcal{P}_{B^\top}A^\top(AA^\top)^{-1}\eps\|
\end{align}

\noindent
By Lemma \ref{ba-helper}, there exists constant $c_2>0$ such that $\|\mathcal{P}_{B^\top} \mathcal{P}_{A^\top} \beta\| \le 1.5\sqrt{\frac{n}{p}}\|\beta\|$ and $\| \mathcal{P}_{B^\top}A^\top(AA^\top)^{-1}\eps\| \le 1.5\sqrt{\frac{n}{p}}\|A^\top(AA^\top)^{-1}\eps\|$ with probability at least $1-2e^{-c_2n}$ each. By Lemma \ref{pd-helper}, $\|A^\top(AA^\top)^{-1}\eps\| \le \frac{\sqrt{n}\|\theta\|}{\sqrt{p\gam}(\sqrt{p-d}-2\sqrt{n})}$ (failure probability already accounted for on Event E). This gives the following bound with probability $1-4e^{-c_2n}$:

\begin{align}
    II &\le 21\sqrt{n\gam}\|\theta\|\|\beta\| + 21\frac{n}{\sqrt{p}(\sqrt{p-d}-2\sqrt{n})}\|\theta\|^2
\end{align}

\noindent
Substituting $\|\beta\| = \frac{\sqrt{p\gam}}{p\gam+1}\|\theta\|$ gives

\begin{align}
    II &\le 21\frac{\gam\sqrt{np}}{p\gam+1}\|\theta\|^2 + 21\frac{n}{\sqrt{p}(\sqrt{p-d}-2\sqrt{n})}\|\theta\|^2
\end{align}

\noindent
Using $p\gam+1 > p\gam$ gives

\begin{align}
    II &\le 21\frac{\sqrt{n}}{\sqrt{p}}\|\theta\|^2 + 21\frac{n}{\sqrt{p}(\sqrt{p-d}-2\sqrt{n})}\|\theta\|^2
\end{align}

\textbf{Lastly we bound term $III$}. Using $p\gam+1 > p\gam$ gives the following bound:

\begin{align}
    III \le \frac{12}{p\gam}\|\theta\|^2
\end{align}

\textbf{Putting all three terms together} gives the following bound with probability $1-10e^{-cd}$ where $c=min(c_1, c_2)$:

\begin{align}
    R(f_{\hbba}) - R(f_{\hba}) &\le 12\frac{\sqrt{d}}{\sqrt{p}}\|\theta\|^2 + 12\frac{\sqrt{nd}}{\sqrt{p}(\sqrt{p-d}-2\sqrt{n})} \|\theta\|^2 + 21\frac{\sqrt{n}}{\sqrt{p}}\|\theta\|^2 \\&\qquad+ 21\frac{n}{\sqrt{p}(\sqrt{p-d}-2\sqrt{n})}\|\theta\|^2 + \frac{12}{p\gam}\|\theta\|^2 \\
    &= \frac{12\sqrt{d} +21\sqrt{n}}{\sqrt{p}}\|\theta\|^2 + \frac{12\sqrt{nd} + 21n}{\sqrt{p}(\sqrt{p-d}-2\sqrt{n})} \|\theta\|^2 + \frac{12}{p\gam}\|\theta\|^2
\end{align}

Using $d \le n$ gives the following bound:

\begin{align}
    R(f_{\hbba}) - R(f_{\hba})
    &\le \left(\frac{ 33\sqrt{n}}{\sqrt{p}} + \frac{ 33n}{\sqrt{p}(\sqrt{p-n}-2\sqrt{n})}  + \frac{12}{p\gam}\right)\|\theta\|^2
\end{align}

By assumption, $p \ge 17n$, so $\frac{n}{\sqrt{p}(\sqrt{p-n}-2\sqrt{n})} \le \frac{\sqrt{n}}{\sqrt{p}}$, which gives the following bound:

\begin{align}
    R(f_{\hbba}) - R(f_{\hba})
    &\le \left(\frac{ 66\sqrt{n}}{\sqrt{p}} + \frac{12}{p\gam}\right)\|\theta\|^2
\end{align}

\end{proof}

\begin{lemma} \label{w-lem}
    Let $W \in \R^{p \times d}$. Then 
    
    \begin{align}
        (I+WW^\top)^{-1}W = W(W^\top W+I)^{-1}
    \end{align}
    
\end{lemma}

\begin{proof}

Let $W=USV$ be the SVD of $W$ where $U \in \R^{p \times d}, S \in \R^{d \times d}, V \in \R^{d \times d}$. Then we have

\begin{align}
    (I+WW^\top)^{-1}W &= (I+USVV^\top SU^\top)^{-1}USV \\
    &= (I+US^2U^\top)^{-1}USV
\end{align}

Let $\tilde{U} \in \R^{p \times p}$ have the first $d$ columns be $U$ and the last $p-d$ columns be the rest of the orthonormal basis. Then we have

\begin{align}
    (I+WW^\top)^{-1}W &= (\tilde{U}\tilde{U}^\top + US^2U^\top)^{-1}USV \\
    &= (\tilde{U}\tilde{U}^\top + \tilde{U} \begin{bmatrix}
S^2 & 0 \\
0 & 0 
\end{bmatrix} \tilde{U}^\top)^{-1}USV \\
&= \left(\tilde{U} \begin{bmatrix}
I+S^2 & 0 \\
0 & I 
\end{bmatrix} \tilde{U}^\top\right)^{-1}USV \\
&= \tilde{U} \left(\begin{bmatrix}
I+S^2 & 0 \\
0 & I 
\end{bmatrix}\right)^{-1} \tilde{U}^\top USV \\
&= \tilde{U} \left(\begin{bmatrix}
I+S^2 & 0 \\
0 & I 
\end{bmatrix}\right)^{-1} (I_d, 0)^\top SV \\
&= \tilde{U}((I+S^2)^{-1}, 0)^\top SV \\
&= U(I+S^2)^{-1}SV \\
&= US(I+S^2)^{-1}V \\
&= USVV^\top(I+S^2)^{-1}V \\
&= W(I+W^\top W)^{-1}
\end{align}

\end{proof}

\end{document}